\documentclass[letterpaper]{article} 
\usepackage{aaai2027}  
\usepackage[hyphens]{url}  
\usepackage{graphicx} 
\usepackage{natbib}  
\usepackage{caption} 
\usepackage{algorithm}
\usepackage{algpseudocode}
\usepackage{times}

\usepackage{latexsym}
\usepackage{amssymb}
\usepackage{amsmath}
\usepackage{amsthm}
\usepackage{booktabs}
\usepackage{enumitem}
\usepackage{graphicx}
\usepackage{xcolor}
\usepackage{array}
\usepackage{multirow}
\usepackage{amsfonts}
\usepackage[dvipsnames]{xcolor}
\usepackage{tikz}
\usepackage{pgfplots}
\usetikzlibrary{pgfplots.statistics}
\usepackage{subcaption}
\usepackage{xspace}
\usepackage{caption}
\usepackage{makecell}
\usepackage{xcolor}
\usepackage{listings}
\usepackage{newfloat}
\DeclareFloatingEnvironment[
   fileext=lol,
   name=Listing
]{listing}

\usepackage{subcaption}
\DeclareCaptionSubType{listing}

\lstdefinelanguage{pddl}
{
    keywords={
        define,
    },
    keywordstyle=[2]{\color{red}},
    morekeywords={[2]
        domain,
        problem,
        requirements,
        predicates,
        types,
        objects,
        action,
        init,
        goal,
    },
    keywordstyle=[3]{\color{MidnightBlue}},
    morekeywords={[3]
        parameters,
        vars,
        precondition,
        effect,
    },
    keywordstyle={[4]\color{OliveGreen}},
    morekeywords={[4]
        forall,
        or,
        and,
        not,
        when,
        =,
        exists,
        imply,
        oneof
    },
    comment=[l]{\;},
    sensitive=true
}[keywords, comments]

\lstdefinelanguage[mapl]{pddl}[]{pddl}
{
    morekeywords={[2]
      sensor,
    },
    morekeywords={[3]
      sense,
      replan
    },
    morekeywords={[4]
      KIF,
    },
}

\usepackage{newfloat}
\usepackage{listings}
\DeclareCaptionStyle{ruled}{labelfont=normalfont,labelsep=colon,strut=off} 
\floatstyle{ruled}
\newfloat{listing}{tb}{lst}{}
\floatname{listing}{Listing}
\title{Planning Task Shielding: Detecting and Repairing Flaws in Planning Tasks through Turning them Unsolvable}

\author{
    Alberto Pozanco, Pietro Totis, Marianela Morales, Daniel Borrajo
}
\affiliations{
    J.P. Morgan AI Research\\\{alberto.pozancolancho,marianela.moraleselena\}@jpmorgan.com, \{name.surname\}@jpmorgan.com}

\usepackage{bibentry}

\begin{document}

\newcommand{\marianela}[1]{\textcolor{orange}{#1}}
\newcommand{\pietro}[1]{\textcolor{violet}{#1}}
\newcommand{\alberto}[1]{\textcolor{blue}{#1}}
\newcommand{\daniel}[1]{\textcolor{red}{#1}}

\newcommand{\fluents}{\ensuremath{{\cal F}}\xspace}
\newcommand{\actions}{\ensuremath{{\cal A}}\xspace}
\newcommand{\init}{\ensuremath{{\cal I}}\xspace}
\newcommand{\goal}{\ensuremath{{\cal G}}\xspace}
\newcommand{\cost}{\ensuremath{c}\xspace}
\newcommand{\costfunction}{\ensuremath{C}\xspace}
\newcommand{\task}{\ensuremath{{\cal P}}\xspace}
\newcommand{\plan}{\ensuremath{\pi}\xspace}
\newcommand{\plans}{\ensuremath{\Pi}\xspace}
\newcommand{\stripstask}{\ensuremath{\task=\langle \fluents, \actions, \init, \goal \rangle}\xspace}
\newcommand{\state}{\ensuremath{s}\xspace}
\newcommand{\allstates}{\ensuremath{\cal S}\xspace}
\newcommand{\reachablestates}{\ensuremath{{\cal S}_R}\xspace}
\newcommand{\fluent}{\ensuremath{f}\xspace}
\newcommand{\action}{\ensuremath{a}\xspace}
\newcommand{\name}{\ensuremath{\textsc{name}}\xspace}
\newcommand{\precondition}{\ensuremath{\textsc{pre}}\xspace}
\newcommand{\effects}{\ensuremath{\textsc{eff}}\xspace}
\newcommand{\addeffects}{\ensuremath{\textsc{add}}\xspace}
\newcommand{\deleffects}{\ensuremath{\textsc{del}}\xspace}
\newcommand{\actionapplication}{\ensuremath{\gamma}\xspace}
\newcommand{\planapplication}{\ensuremath{\Gamma}\xspace}
\newcommand{\alternativeplans}[1]{\ensuremath{\plans^{#1}}\xspace}
\newcommand{\alternativeplan}{\ensuremath{\plan^{\mathsf{a}}}\xspace}

\newcommand{\fluentsticky}{\ensuremath{F_{c}}\xspace}
\newcommand{\goalhat}{\bar{\ensuremath{G}}\xspace}
\newcommand{\subgoal}{\ensuremath{g}\xspace}
\newcommand{\Sticky}{\mathsf{commit}\xspace}
\newcommand{\sticky}{\mbox{-}\mathsf{commit}\xspace}
\newcommand{\forcesticky}{\mbox{-}\mathsf{forcecommit}\xspace}
\newcommand{\Forcesticky}{\mathsf{forcecommit}\xspace}
\newcommand{\actionssticky}{\actions^{\mathsf{C}}\xspace}
\newcommand{\actionsnoncommit}{\actions^{\neg\mathsf{C}}\xspace}
\newcommand{\actionnoncommit}{\action^{\neg\mathsf{C}}\xspace}
\newcommand{\actioncommit}{\action^{\mathsf{C}}\xspace}
\newcommand{\actionsaddgoals}{\actions^\goal\xspace}
\newcommand{\actionsdeletegoals}{\actions^{\neg\goal}\xspace}
\newcommand{\actionslonely}{\actions^\ensuremath{L}\xspace}
\newcommand{\stickytask}{\ensuremath{{\cal P}_{\mathsf{c}}}\xspace}
\newcommand{\actionaddgoals}{\action^\goal\xspace}
\newcommand{\actiondeletegoals}{\action^{\neg\goal}\xspace}
\newcommand{\commitgoals}{\ensuremath{C}_{\actionaddgoals}\xspace}
\newcommand{\noncommitgoals}{\ensuremath{C}_{\actiondeletegoals}\xspace}

\newcommand{\atpackage}[3]{\textsf{at}(\mathsf{#1\mbox{-}#2}\mbox{ }\mathsf{#3})}
\newcommand{\at}[2]{\textsf{at}(\mathsf{#1}\mbox{ }\mathsf{#2})}
\newcommand{\package}[1]{\textsf{package}\textsf{#1}}
\newcommand{\drivetruck}[2]{\textsf{drivetruck}\mbox{ }\mathsf{#1}\mbox{ }\mathsf{#2}}
\newcommand{\loadtruck}[2]{\textsf{loadtruck}\mbox{ }\mathsf{#1}\mbox{-}\mathsf{#2}}
\newcommand{\unloadtruck}[2]{\textsf{unloadtruck}\mbox{ }\mathsf{#1}\mbox{-}\mathsf{#2}}

\newtheorem{definition}{Definition}
\newtheorem{proposition}{Proposition}
\newtheorem{lemma}{Lemma}
\newtheorem{corollary}{Corollary}
\newtheorem{example}{Example}
\newtheorem{remark}{Remark}

\newcommand{\lazytask}{\task_{\mathsf{D}}}
\newcommand{\lazyfluents}{ \fluents_{\mathsf{D}}}
\newcommand{\lazyactions}{\actions_{\mathsf{D}}}
\newcommand{\lazyinit}{\init_{\mathsf{D}}}
\newcommand{\lazygoal}{\goal_{\mathsf{D}}}

\newcommand{\eagertask}{\task_{\cal{D}}}

\newcommand{\lamafirst}{\ensuremath{\textsc{lamaF}}\xspace}

\newcommand{\plandisruption}{\ensuremath{{\cal D}}\xspace}

\newcommand{\shielding}{shielding\xspace}
\newcommand{\Shielding}{Shielding\xspace}

\newcommand{\allplansmin}{\textsc{allmin}\xspace}
\newcommand{\compilation}{\textsc{aga}\xspace}
\newcommand{\dfs}{\textsc{dfs}\xspace}
\newcommand{\bfs}{\textsc{bfs}\xspace}
\newcommand{\mcts}{\textsc{mcts}\xspace}
\newcommand{\setofmodifications}{\ensuremath{{\cal M}}\xspace}
\newcommand{\modification}{\ensuremath{{\mu}}\xspace}

\definecolor{solvability}{HTML}{8FAADC}
\newcommand{\solvability}[1]{\textcolor{solvability}{#1}}
\definecolor{repair}{HTML}{F4B183}
\newcommand{\repair}[1]{\textcolor{repair}{#1}}
\definecolor{test}{HTML}{FFD966}
\newcommand{\test}[1]{\textcolor{test}{#1}}
\definecolor{unsolvable}{HTML}{A9D18E}
\newcommand{\unsolvable}[1]{\textcolor{unsolvable}{#1}}

\newcommand{\gab}{\textsc{Goal Achiever Blocking}\xspace}

\maketitle

\begin{abstract}
Most research in planning focuses on generating a plan to achieve a desired set of goals.
However, a goal specification can also be used to encode a property that should never hold, allowing a planner to identify a trace that would reach a flawed state.
In such cases, the objective may shift to modifying the planning task to ensure that the flawed state is never reached, in other words, to make the planning task unsolvable.
In this paper we introduce planning task shielding: the problem of detecting and repairing flaws in planning tasks.
We propose three different approaches to solve these tasks, each aiming to minimize the number of modifications to the original actions that render the planning task unsolvable.
\end{abstract}

\section{Introduction}
Classical planning is the task of finding a plan, which is a sequence of deterministic actions that, when executed from a given initial state, lead to a state where some given goals are true~\cite{ghallab2004automated}.
Most research in planning focuses on generating plans to solve the given task, assuming such a plan exists.

However, planning can also be applied in the opposite way.
This approach involves formalizing the system and the security property to be verified as a planning task. 
If this planning task is proven to be unsolvable \cite{eriksson2017unsolvability,staahlberg2021learning}, it indicates that the security property is upheld within the system. 
Conversely, if a solution is found, the resulting plan outlines a sequence of steps or actions that could potentially falsify the security property.
This approach to planning has been utilized to identify flaws in cybersecurity systems~\cite{boddy2005course,hoffmann2015simulated} and cryptographic protocols~\cite{pozanco2021proving}.
In these systems, when a flaw is detected, a domain expert reviews the plan that leads to it and manually modifies the system’s dynamics (its actions) to prevent that trace from occurring, hoping this will resolve the issue.
However, addressing the flaw locally may introduce new vulnerabilities elsewhere in the system, potentially resulting in a cycle that is tedious and difficult to resolve.

\begin{listing}
\begin{lstlisting}[language=pddl,numbers=none,numbersep=0,xleftmargin=0.01\parindent]
(:action submit_application
  :parameters ()
  :precondition (documents_submitted)
  :effect (and (application_complete)
               (not (documents_submitted))))

(:action direct_approval
  :parameters ()
  :precondition (application_complete)
  :effect (granted_approval))

(:action escalation
  :parameters ()
  :precondition 
  (and (application_complete) 
       (client_concerns))
  :effect (escalated))
\end{lstlisting}
\caption{Approval workflow described as PDDL actions.}
\label{lst:rule_example}
\end{listing}
We illustrate these type of problems with a simple running example (Listing~\ref{lst:rule_example}), which describes an approval workflow in PDDL~\cite{haslum2019introduction}. 
Since the workflow may have been created by non-experts, it could contain errors and represents a best-effort formalization. 
The set of fluents is $\langle \mathsf{documents\_submitted}$, $\mathsf{application\_complete}$, $\mathsf{granted\_approval}$, $\mathsf{escalated}$, $\mathsf{client\_concerns}$, $\mathsf{safe\_client}\rangle$.
The $\mathsf{submit\_application}$ action completes an application if all documents are submitted; $\mathsf{direct\_approval}$ allows direct approval for completed applications, and $\mathsf{escalation}$ handles cases with client concerns. 
We may want to check if flawed states can arise, such as reaching a state where an application is both $\mathsf{granted\_approval}$ and $\mathsf{escalated}$, starting from $\mathsf{documents\_submitted}$ and $\mathsf{client\_concerns}$. 
The plan $\plan = (\mathsf{submit\_application}, \mathsf{escalation}, \mathsf{direct\_approval})$ achieves this, indicating the workflow is ill-defined. 
Addressing this requires modifying actions, such as removing $\mathsf{escalated}$ from the $\mathsf{escalation}$ action effects; or adding $\mathsf{safe\_client}$ as a precondition to $\mathsf{direct\_approval}$.
Selecting the best fix depends on the domain and action semantics, and care is needed, as local changes may introduce new vulnerabilities elsewhere in the workflow.

In this paper, we propose an extension to the traditional approach of identifying flaws in systems represented as planning tasks by introducing the capability to automatically fix these flaws.
Our method focuses on making the planning task unsolvable, in contrast to domain repair works~\cite{gragera2023planning,lin2023towards, bercher2025survey,gragera2025gains}, which aim to modify the planning task to make it solvable.
We refer to this problem as \emph{planning task shielding}, and formally define it as finding the minimal set of precondition additions, add effect removals, and delete effect additions to the original set of actions such that the planning task becomes unsolvable.

We then propose three different approaches to solve these tasks, each offering distinct theoretical guarantees and empirical performance.
The first approach computes all loopless plans that solve the planning task, and then formulates an optimization problem to determine the minimal set of modifications required to invalidate all these plans.
The second approach introduces a family of algorithms that search within the space of possible modifications to render the planning task unsolvable.
The third approach conducts a static analysis of the planning task to identify actions that achieve the goals, and then modifies these actions so that the goals can no longer be achieved.

The remainder of the paper is organized as follows. 
We first formalize classical planning and Monte Carlo Tree Search.
We then provide a formal definition of planning task shielding problems and their solutions, followed by a presentation of our three approaches for solving these tasks. 
Next, we evaluate the proposed algorithms on existing and novel benchmarks, highlighting the trade-offs between their theoretical guarantees and empirical performance.
Finally, we conclude by discussing the main results and outlining potential directions for future research.
\section{Background}
\subsection{Classical Planning}
We formally define a planning task as follows:
\begin{definition}[Planning task]\label{def:strips-plan-task}
  A planning task
is a tuple \stripstask, where \fluents is a set of fluents, \actions is a set of
 actions, $\init \subseteq \fluents$ is an initial state, and $\goal\subseteq \fluents$ is a goal specification.  
\end{definition}

A state $\state \subseteq \fluents$ is a set of fluents that are true at a given time.
A state $\state \subseteq \fluents$ is a goal state iff $\goal \subseteq \state$.
Each action $\action \in \actions$ is characterized by the following components.
A set of preconditions $\precondition(\action)$, which are set of fluents that need to be true for the action to be applied. 
Add and delete effects $\addeffects(\action)$ and $\deleffects(\action)$, which are set of fluents that are added (or deleted) once the action is applied.
We assume $\addeffects(\action) \cap \deleffects(\action) = \emptyset$.
And finally a cost $\cost(\action) \in \mathbb{R}$ associated with performing the action.
An action \action is applicable in a state \state iff $\precondition(\action)\subseteq~s$.
We define the result of applying an action in a state as $\actionapplication(\state,\action)=(\state \setminus \deleffects(\action)) \cup \addeffects(\action)$.
A sequence of actions $\plan=(\action_1,\ldots,\action_n)$ is applicable in a state $\state_0$ if there are states $(\state_1.\ldots,\state_n)$ such that $\action_i$ is applicable in $\state_{i-1}$ and $\state_i=\actionapplication(\state_{i-1},\action_i)$.
The resulting state after applying a sequence of actions is $\planapplication(\state,\pi)=\state_n$, and $\cost(\plan) = \sum_{i}^n \cost(\action_i)$ denotes the cost of $\plan$.
A sequence of actions is simple if it does not traverse the same state more than once.
The solution to a planning task $\task$ is a plan, i.e., a sequence of actions $\plan$ such that $\goal \subseteq \planapplication(\init,\pi)$.
We denote as $\Pi(\task)$ the set of all simple solution plans to planning task $\task$.
A plan with minimal cost is optimal.
A planning task is unsolvable if there is no sequence of actions $\plan$ such that $\goal \subseteq \Gamma(\init, \plan)$.

\subsection{Monte Carlo Tree Search}\label{sec:mcts}
Monte Carlo Tree Search (MCTS)~\cite{mcts_survey,mcts_survey_recent} incrementally builds a search tree over large decision spaces, with states as nodes and actions as edges, and is particularly effective where exhaustive enumeration is infeasible~\cite{MCTS_app}, for example in games~\cite{mcts_general_game_playing,alphago}. Each iteration follows four steps: a tree policy selects a path to a leaf (selection); a new child is created for a chosen action (expansion); a default policy estimates the leaf's value (simulation); and ancestor statistics are updated accordingly (backpropagation). A popular tree policy is UCT~\cite{UCT}, which scores action $a$ in state $s$ as $Q(s,a) + C\cdot\sqrt{\frac{\ln N(s)}{N(s,a)}}$, where $Q(s,a)$ is the reward estimate, $N(s)$ is the visit count for $s$, and $N(s,a)$ is the count for action $a$ in $s$; the constant $C$, typically $\sqrt{2}$ when $Q\in [0,1]$~\cite{mcts_survey_recent}, controls the trade-off between exploiting high-reward actions and exploring less-visited ones.
When the branching factor is large, \emph{progressive widening}~\cite{mcts_pw2,mcts_pw} limits expansion to a growing subset of actions, introducing a new one only when $\lfloor N(s)^\alpha \rfloor \ge |\mathit{children}(s)|$ for $\alpha\in(0,1)$. 

\section{Shielding Planning Tasks}
Given a system formalized as a planning task, where the goal specifies a property that should never be satisfied, our objective is to \emph{shield} the system. 
This means modifying the planning task so that these modifications render it unsolvable.
Although planning tasks can be made unsolvable by modifying $\init$, $\goal$, or $\actions$, we focus exclusively on the latter. 
\begin{remark}
    Note that by limiting ourselves to modifying $\actions$, we are unable to turn unsolvable planning tasks where $\goal \subseteq \init$, i.e., those tasks where the empty plan $\pi=\emptyset$ is already a solution.
\end{remark}
To formalize this approach, we first define what constitutes an atomic action modification.

\begin{definition}[Action Modification]\label{def:action-modification}
    Given a planning task $\stripstask$, an action modification $\mu$ maps an action $\action \in \actions$ to an action $\mu(\action) = \action'$ obtained by adding or removing a single fluent $f \in \fluents$ in exactly one of $\precondition(\action)$, $\addeffects(\action)$, or $\deleffects(\action)$.
\end{definition}
We write $\setofmodifications^*(\task)$ for the set of all action modifications over the actions in $\actions$, abbreviated $\setofmodifications^*$ when $\task$ is clear from context.
\begin{remark}
We assume that applying several action modifications to an action $\action$ yields a single well-defined modified action $\action'$. In particular, the final action $\action'$ does not depend on the order in which the modifications are applied.
\end{remark}

With these notions in place, we can now define a shielding planning task and its solution.
\begin{definition}[Shielding Planning Task]
    A shielding planning task $\cal S$ is a tuple ${\cal S} = \langle \task, \setofmodifications \rangle$, where $\stripstask$ is a planning task and $\setofmodifications \subseteq \setofmodifications^*$ is a set of allowed modifications.
\end{definition}

\begin{definition}[Shielding Solution]\label{def:shielding-solution} A solution to a shielding planning task ${\cal S} = \langle \task, \setofmodifications \rangle$ is a set of action modifications $M \subseteq \setofmodifications$ that induces a new set of actions $\actions^\prime$, such that the resulting planning task $\task_{\cal S} = \langle \fluents, \actions^\prime, \init, \goal \rangle$ is unsolvable.
\end{definition}

Although trivial modifications are possible, our aim is to identify the minimal set of action modifications that render the planning task unsolvable.
We formalize this notion of optimality as follows:
\begin{definition}[Shielding Solution Optimality]\label{def:optimality}
A Shielding Solution $\actions^\prime$ obtained with modifications $M$ is optimal iff there does not exist another solution $\actions^{\prime\prime}$ obtained with modifications $M'$ and $|M'| < |M|$. 
\end{definition}

In order to compute optimal (or high quality) solutions for a shielding planning task, we do not need to reason about all the possible ways in which an action can be modified.
In particular, we can restrict ourselves to the set of action modifications that reduce the number of plans that solve the original planning task.

\begin{proposition}[Monotonic Decrease of Solution Plans]\label{prop:decrease_plans}
Let $\task=\langle \fluents, \actions, \init, \goal\rangle$ be a planning task and let $\task'=\langle \fluents, \actions', \init, \goal\rangle$ be obtained from $\task$ by modifying actions only by:
\begin{itemize}
    \item adding preconditions, $\precondition(\action)\subseteq \precondition(\action')$
    \item removing add effects, $\addeffects(\action')\subseteq \addeffects(\action)$
    \item adding delete effects, $\deleffects(\action)\subseteq \deleffects(\action')$
\end{itemize}
Then, after applying any modification, the number of plans that solve the task decreases monotonically, i.e., $\lvert \Pi(\task')\rvert \le \lvert \Pi(\task)\rvert$.
\end{proposition}
\begin{proof}
Let $s\subseteq \fluents$ be any state, and let $\action'\in \actions'$ correspond to $\action\in \actions$. Suppose $\action'$ is applicable in $s$ in $\task'$. Then, by definition, $\precondition(\action')\subseteq s$. Since $\precondition(\action)\subseteq \precondition(\action')$, it follows that $\precondition(\action)\subseteq s$, so $\action$ is also applicable in $s$ in $\task$. Thus, any action applicable in $\task'$ at a state is also applicable in $\task$ at the same state. Conversely, suppose $\action$ is applicable in $s$ in $\task$ (i.e., $\precondition(\action)\subseteq s$), if there exists $p\in \precondition(\action')\setminus \precondition(\action)$ with $p\notin s$, then $\action'$ is not applicable in $s$ in $\task'$. Therefore, the set of applicable actions in any state in $\task'$ is a subset of those in $\task$.

For any state $s$ where $\action'$ is applicable, the resulting state after applying $\action'$ in $\task'$ is $s' = (s \setminus \deleffects(\action')) \cup\addeffects(\action')$. Since $\addeffects(\action')\subseteq \addeffects(\action)$ and $\deleffects(\action)\subseteq \deleffects(\action')$, it follows that $s \setminus \deleffects(\action')\subseteq s \setminus \deleffects(\action)$ and $(s \setminus \deleffects(\action')) \cup\addeffects(\action')\subseteq (s \setminus\deleffects(\action)) \cup\addeffects(\action)$. Thus, the state reached by applying $\action'$ in $\task'$ is a subset of the state reached by applying $\action$ in $\task$.

Consider any plan $\plan' = (\action'_1, \ldots, \action'_n)$ that solves $\task'$. For each $i$, let $\action_i$ be the corresponding action in $\actions$. By the above, $\plan = (\action_1, \ldots, \action_n)$ is also executable in $\task$ from $\init$, and the sequence of states reached in $\task$ contains those reached in $\task'$. Since the goal is the same, if $\plan'$ reaches $\goal$ in $\task'$, then $\pi$ also reaches $\goal$ in $\task$.

Therefore, every plan that solves $\task'$ also solves $\task$, i.e., $\Pi(\task') \subseteq \Pi(\task)$, and thus $\lvert \Pi(\task')\rvert \le \lvert \Pi(\task)\rvert$.
\end{proof}

In addition to restricting the set of modifications, we can also limit the subset of actions that need to be modified to render $\task$ unsolvable.
Let $\actions^{\Pi} = \{ \action \mid \action \in \plan, \forall \plan \in \Pi(\task) \}$ denote the set of actions that appear in at least one plan solving the original task.
\begin{remark}\label{remark:mod_actions}
    It suffices to consider modifications to actions in $\actions^{\Pi}$ (and not the remaining actions in $\actions$) to render $\task$ unsolvable.
\end{remark}

Next, we present our three different approaches to solve shielding planning tasks.

\section{\allplansmin}
\allplansmin follows a two-step process: (1) it computes all the simple (loopless) plans that can solve the original planning task, $\Pi(\task)$; and (2) it determines the set of minimal modifications to the original actions $\actions$ that would block the execution of all the plans in $\Pi(\task)$.

To compute the set of plans that solve the original planning task, we can utilize any planner capable of generating not just a single plan, but a set of plans for a given task~\cite{katz2018novel,speck-et-al-aaai2020,speck2025counting}.
We will provide further details about the specific tool used in our experimental setup.

We formulate the task of computing the minimal modifications to the original actions that would block the execution of all the plans as a Mixed-Integer Linear Program (MILP).
The MILP receives as input the planning task $\task$, and the set of plans that solve it $\Pi(\task)$.
We reduce the number of variables and constraints needed by leveraging Proposition~\ref{prop:decrease_plans} and Remark~\ref{remark:mod_actions}.
We add a fake action $a^g$ to $\actions$, which represents the achievement of the goals $\goal$.
$\precondition(a^g) = \goal$ and $\addeffects=\deleffects=\emptyset$.
We append this action to each plan $\pi \in \Pi(\task)$.
We have the following set of decision variables:
\begin{itemize} 
\item $\mathrm{pre}_{a,f} \in \{0,1\}$: 1 if fluent $f$ is added as a precondition to action $a$. 
\item ${\mathrm{add}}_{a,f} \in \{0,1\}$: 1 if fluent $f$ is removed from the add effects of action $a$. 
\item ${\mathrm{del}}_{a,f} \in \{0,1\}$: 1 if fluent $f$ is added to the delete effects of action $a$. 
\item $s_{\pi,i,f} \in \{0,1\}$: 1 if fluent $f$ holds after step $i$ in plan $\pi$. 
\item $\mathrm{enabled}_{\pi,i} \in \{0,1\}$: 1 if the action at step $i$ in plan $p$ is executable. 
\item $\mathrm{pre\_unsat}_{\pi,i,f} \in \{0,1\}$: 1 if precondition $f$ is present and not satisfied at step $i$ in plan $\pi$. 
\end{itemize}

The objective is to minimize the total number of action modifications, specifically the addition of new preconditions, removal of add effects, and addition of delete effects:

\begin{align}
\min \quad &
\sum_{a \in \actions^{\Pi},\, f \in \fluents \setminus \mathrm{pre}(a)} {\mathrm{pre}}_{a,f}
\\
&\quad +\;
\sum_{a \in \actions^{\Pi},\, f \in \mathrm{add}(a)} {\mathrm{add}}_{a,f}
\\
&\quad +\;
\sum_{a \in \actions^{\Pi},\, f \in \fluents \setminus \{\mathrm{del}(a) \cup \mathrm{add}(a)\}} {\mathrm{del}}_{a,f}
\end{align}

The constraints ensure the correct propagation of fluents, satisfaction of preconditions, and blocking of plans:

\begin{enumerate} \item \textbf{Initial State:} The initial value of each fluent for each plan is set according to the initial state $\init$. We include the following constraints for each $\plan \in \Pi(\task)$ and $f \in \fluents$:
\begin{align}
s_{\pi,0,f} = 1 \quad \text{if} \quad f \in \init
\end{align}
\begin{align}
s_{\pi,0,f} = 0 \quad \text{if} \quad f \notin \init
\end{align}

\item \textbf{Precondition Satisfaction:} For each action in each plan, the variable $\mathrm{pre\_unsat}_{\pi,i,f}$ captures whether precondition $f$ of action $a_i$ is unsatisfied, considering both original and newly added preconditions. We include the following constraints for each $\plan \in \Pi(\task)$, $a_i \in \plan$, and $f \in \fluents$:

\noindent If $f \in \precondition(a)$:
\begin{align}
    \small
    \mathrm{pre\_unsat}_{\pi,i+1,f}  =  1 - s_{\pi,i,f}
\end{align}

\noindent If $f \notin \precondition(a)$:
\begin{align}
    \small
    \mathrm{pre\_unsat}_{\pi,i+1,f} \geq \mathrm{pre}_{a,f} - s_{\pi,i,f} 
\end{align}
\begin{align}
    \small
    \mathrm{pre\_unsat}_{\pi,i+1,f} \leq \mathrm{pre}_{a,f}  
\end{align}
\begin{align}
    \small
    \mathrm{pre\_unsat}_{\pi,i+1,f} \leq 1 - s_{\pi,i,f}  
\end{align}

\item \textbf{Action Enabledness:} An action is enabled if all its preconditions are satisfied. We include the following constraints for each $\plan \in \Pi(\task)$ and $a_i \in \plan$:
\begin{align}
    \mathrm{enabled}_{\plan,i+1} \geq 1 - \sum_{f \in \fluents} \mathrm{pre\_unsat}_{\pi,i,f}
\end{align}

\item \textbf{State Propagation:} The fluents are updated according to the effects of the actions and the modifications.  We include the following constraints for each $\plan \in \Pi(\task)$, $a_i \in \plan$, and $f \in \fluents$:

\noindent If $f \in \addeffects(a)$:
\begin{align}
    \small
    s_{\pi,i+1,f}  \geq  1 - \mathrm{add}_{a,f}
\end{align}

\begin{align}
    \small
    s_{\pi,i+1,f} - s_{\pi,i,f}  \geq  1 - \mathrm{add}_{a,f}
\end{align}

\begin{align}
    \small
    s_{\pi,i,f} - s_{\pi,i+1,f}  \geq  1 - \mathrm{add}_{a,f}
\end{align}

\noindent Else If $f \in \deleffects(a)$:
\begin{align}
    \small
    s_{\pi,i+1,f} = 0
\end{align}

\noindent Else:
\begin{align}
    \small
    s_{\pi,i+1,f}  \leq  1 - \mathrm{del}_{a,f}
\end{align}

\begin{align}
    \small
    s_{\pi,i+1,f} - s_{\pi,i,f}  \leq  1 - \mathrm{del}_{a,f}
\end{align}

\begin{align}
    \small
    s_{\pi,i,f} - s_{\pi,i+1,f}  \leq  1 - \mathrm{del}_{a,f}
\end{align}

\begin{align}
    \small
    s_{\pi,i+1,f} - s_{\pi,i,f}  \leq  0
\end{align}

\begin{align}
    \small
    s_{\pi,i,f} - s_{\pi,i+1,f}  \leq  0
\end{align}

\item \textbf{Plan Blocking:} To ensure every plan is blocked, at least one action in each plan must not be enabled. We include the following constraints for each $\pi \in \Pi(\task)$:
\begin{align}
    \small
    \sum_{i=1}^{|\pi|} \mathrm{enabled}_{\pi,i} \leq |\pi| - 1
\end{align}

\item \textbf{Goal Persistence}: We impose the following restriction to enforce that the goal is not modified, i.e., the preconditions of $a^g$ cannot be modified.
\begin{align}
    \sum_{f \in \fluents \setminus \goal} \mathrm{pre}_{a^g,f} = 0 
\end{align}

\end{enumerate}

This MILP identifies the minimal set of action modifications needed to block all plans. By incorporating new preconditions ($\mathrm{pre}_{a,f}=1$), adding delete effects ($\mathrm{del}_{a,f}=1$), and removing specified add effects ($\mathrm{add}_{a,f}=1$), the resulting set of actions $\actions^\prime$ ensures that the original planning task becomes unsolvable. 
As a result, $\actions^\prime$ serves as a shielding solution for the original planning task $\task$.

\section{Search in the Space of Modifications}\label{sec:search}

Our second approach frames the shielding problem as a search over the space of action modifications.
A \emph{search state} is a set of action modifications applied cumulatively to $\actions$, yielding a modified action set $\actions'$.
A \emph{search action} is a single atomic modification from Definition~\ref{def:action-modification}, and the  goal is to reach a state where the resulting planning task is unsolvable.
Even for small benchmarks, the pool of candidate modifications is large: every action contributes one candidate per fluent in each of its three components, and grounded domains can have many actions and fluents.
Exhaustive enumeration is therefore infeasible: we reduce the effective branching factor through a shared action-filtering mechanism inspired by counterexample-guided search~\cite{seipp2018counterexample}.
At each state, the planner is called on the current modified domain and if the task is unsolvable, the search has found a solution.
Otherwise, the returned plan is a counterexample: it witnesses that the domain is still solvable and identifies precisely which actions enable it.
The available modifications for the next step are therefore filtered to those that (a) target an action appearing in the current plan, and (b) actually invalidate that plan when simulated.
This reduces the effective branching factor substantially and focuses the search on modifications that are structurally relevant.

\textbf{Breadth-first search.} \bfs exhaustively tests all states reachable with exactly $k$ modifications before considering any state requiring $k{+}1$.
This guarantees that the first solution found uses the minimum number of modifications, i.e., \bfs returns an optimal shielding solution in the sense of Definition~\ref{def:optimality}.
The cost is a branching factor that grows combinatorially with depth, making \bfs impractical beyond the first few levels on larger domains.

\textbf{Depth-first search.} \dfs follows a single branch to the specified maximum depth before backtracking.
It typically reaches a solution faster than \bfs when one exists at moderate depth, but offers no optimality guarantee.
Moreover, for some domains the solution may be at a higher depth than the given maximum.
More fundamentally, DFS commits fully to the choices made along its current branch: without any signal to distinguish which modifications are likely to render the task unsolvable, it may spend most of its budget exploring an unproductive region of the space and only redirect its effort after exhausting it entirely.

\textbf{Monte Carlo Tree Search.} \mcts addresses the lack of guidance in \dfs by using the simulation score to direct exploration toward promising regions of the search space.
We use UCT with progressive widening: $\alpha$ controls the trade-off between depth and breadth, a lower value keeps the tree narrow and deep, which is effective when solutions require combining several modifications, while a higher value widens the tree earlier, favouring search over combinations of fewer modifications.  

The simulation step is the planner call and the counterexample-guided filtering described above, which determines the available modifications for future expansions. 
Unsolvable nodes receive a score of $1$.
For solvable nodes, the simulation score encodes the hypothesis that longer plans and longer solve times signal a domain that is harder to solve, and thus closer to becoming unsolvable.
We make no claim that this score is a reliable predictor of unsolvability: there is no domain-independent proxy for how close a modified domain is to becoming unsolvable.
The score is a heuristic, computed for free from information the planner already returns, and its potential value lies in providing \emph{some} signal to UCT rather than none.
On domains with many solutions, solve times at shallow depths are nearly identical; plan length provides a more discriminating signal and is therefore the primary component, with solve time acting as a tiebreaker within the same length bucket.
Both are min-max normalised across all nodes seen so far:
\begin{equation}\label{eq:score}
  s_l = \frac{l - l_{\min}}{l_{\max} - l_{\min}}, \qquad
  s_t = \frac{t - t_{\min}}{t_{\max} - t_{\min}},
\end{equation}
where $l$ and $t$ are the plan length and solve time of the node.
Formally, letting $b_l = \min\!\left(9,\, \lfloor s_l \cdot 10 \rfloor\right)$, the base score is:
\begin{equation}\label{eq:base-score}
  \mathit{score} = \frac{b_l}{10} + s_t \cdot 0.09.
\end{equation}
This ensures plan length dominates: no amount of solve-time difference can move a node across a length bucket boundary.
When a new extreme value updates $l_{\min}$, $l_{\max}$, $t_{\min}$, or $t_{\max}$, all nodes in the tree are rescored to maintain a consistent normalisation.
In practice, the search explores at most a few hundred nodes, so the cost of rescoring is negligible compared to individual planner calls.
\section{Action Goal Achievers Analysis}
Our third and final approach for solving shielding planning tasks is based on analyzing \emph{action goal achievers}.

The approach is described in Algorithm~\ref{alg:AGA}.
It first considers only those goal fluents that are not already true in the initial state, i.e., $\goal \setminus \init$. For each $\subgoal \in \goal \setminus \init$, we calculate the set of \emph{achievers}, i.e., the actions where $\subgoal$ appears in the add effects.
The algorithm then selects the $\subgoal$ with the smallest number of action achievers and repairs the task by deleting exactly the corresponding add effects that could achieve $\subgoal$, leaving all other aspects of each action unchanged.

Intuitively, if we prevent the achievement of at least one goal $\subgoal \in \goal$, then the entire goal $\goal$ becomes unreachable, rendering the modified task unsolvable. 
Although this greedy selection of $\subgoal$ may not be optimal with respect to minimizing the number of modifications, it enables us to make the planning task unsolvable without performing a search, unlike previous approaches.

\begin{algorithm}[t]
    \begin{algorithmic}[1]
        \renewcommand{\algorithmicrequire}{\textbf{Input:}}
        \Require Task $\task=\langle\fluents,\actions,\init,\goal\rangle$ 
        \State $m \gets +\infty$
        \ForAll{$g \in \goal \setminus \init$}
          \State $A \gets \{a\in\actions \mid g\in \addeffects(a)\}$
          \If{$|A| < m$}
            \State $g^\star \gets g$;\;\; $A^\star \gets A$ ;\;\; $m \gets |A|$
          \EndIf
        \EndFor
        \ForAll{$a \in A^\star$}
          \State $\addeffects(a) \gets \addeffects(a)\setminus\{g^\star\}$
        \EndFor
        \State \Return $\task'$
    \end{algorithmic}
    \caption{Algorithm $\mathsf{AGA}$}
    \label{alg:AGA}
\end{algorithm}
\section{Evaluation}
\subsection{Experimental Setting}
\paragraph{Benchmark.} 
We use two benchmarks to evaluate \allplansmin.
In the first one, we use a synthetic benchmark (\textsc{synth}) where we generate planning tasks in the form of a graph where we control: the number of plans ($8,16,32,64$), their maximum ($4,8,16,32$) and minimum ($2,4,6,8$) length, and the percentage of plans that share some edges with other plans ($0.4$).
The numbers of fluents, actions, and states range from a few dozen in the smaller instances with 8 plans to a few thousand in the larger instances with $64$ plans.
We generate $10$ random problems for each of the $4$ combinations, giving us a total of $40$ problems of increasing complexity.
The reason we generated this synthetic benchmark is threefold. 
First, it allows us to control the number of plans that solve the planning task, which is not possible when working with existing tasks or using generators for known domains~\cite{torralba2021automatic}. 
Second, tasks in current benchmarks are typically designed to be challenging for planners, often resulting in tasks that are too large to serve as meaningful shielding tasks. 
Finally, in practice, we would expect systems to have only a few ways of reaching a flawed state, rather than the hundreds of thousands of plans typically found in most planning tasks from existing benchmarks~\cite{speck-et-al-aaai2020}. 
For the second benchmark, we chose the $10$ smallest problems from the following domains in the Fast Downward~\cite{helmert2006fast} benchmark collection\footnote{https://github.com/aibasel/downward-benchmarks}: \textsc{blocksworld}, \textsc{rovers}, \textsc{satellite}, \textsc{tpp} and \textsc{depot}.
By selecting this concise yet varied set of standard planning tasks, our goal is to demonstrate how the three approaches can be applied to transform unsolvable problems from domains that are well-known within the community.

\paragraph{Approaches and Metrics.} We evaluate \allplansmin, \bfs, \dfs, \mcts and \compilation on the benchmarks described above.
We use the \textsc{symk} planner~\cite{speck-et-al-aaai2020} to compute all simple plans that solve a given planning task~\cite{vontschammer-et-al-icaps2022}. 
Given that computing all simple plans is not feasible for many tasks in the benchmark, we limit \textsc{symk} to output 10000 plans for \allplansmin, and
we also evaluate two variants: $\allplansmin^{10}$ and $\allplansmin^{100}$. 
Instead of computing all simple plans, these variants compute respectively 10 and 100 plans that solve the task.
Problems with more plans than the \textsc{symk} limit lose the effectiveness and optimality guarantees of the modifications computed by the MILP, but may be sufficient to render the task unsolvable nonetheless.
We then solve the resulting MILPs using the HiGHS solver~\cite{hall2019highs} to determine the minimal modifications required to make the planning task unsolvable.
We also test \dfs with different maximum depth (modifications) $d\in\{5,10,15\}$, denoted with $\dfs^{d}$.
We test \mcts with progressive widening with $\alpha = 0.2$.
For each algorithm, we evaluate:
\begin{enumerate}
    \item Completed ($C$): the number of instances returning a set of modifications $M$ within the time and memory bounds.
    \item Success ($S)$: the number of instances returning a set of modifications $M$ that turn the planning task unsolvable.

    \item Time ($s$): the total execution time (in seconds) required to generate a solution. 

    \item $|M|$: the number of modifications in the solution
\end{enumerate}
We validate that the suggested changes turn the planning task unsolvable by calling \textsc{symk} again and verifying there are no plans that solve the reformulated task $\task^\prime$.

\paragraph{Reproducibility.} Experiments were run on an 8-core, 2.8GHz CPU machine with 32GB RAM, with a time limit of $1800$s for each shielding task.

\subsection{Results}  Figure~\ref{fig:success_analysis} summarizes the number of successful modifications for each method. 
\compilation succeeds for all 90 benchmarks and is also the fastest method. 
However, \compilation is suboptimal and uses on average a higher number of modifications in all benchmarks but \textsc{blocksworld} and \textsc{satellite} (Figure~\ref{fig:calendario}).
On synthetic benchmarks, \compilation returns an average of 47.8\% more modifications (w.r.t. the optimal \allplansmin), and up to 9 modifications more on the planning benchmarks w.r.t. the subset solved by the optimal \bfs.

The second best method is $\allplansmin^{100}$ with 49 successes out of 90, representing a balanced compromise between $\allplansmin^{10}$, which returns insufficient modifications in 50 instances, and \allplansmin, which runs out of memory in 43 instances.
The tradeoff is clear: enumerating more plans gives more information to the MILP to find a set of effective modifications, at a cost of more memory and time resources. 

In the majority of the benchmarks, 10 plans are not representative enough to find a set of modifications that invalidates all possible plans, thus $\allplansmin^{10}$ struggles with rendering the whole task unsolvable. 
This behavior is evident in the \textsc{synthetic} benchmarks, which admit many independent paths to the goal. 
While $\allplansmin^{10}$ computes valid shielding solutions for all tasks in the set with $8$ plans (\textsc{synth4}), it does not succeed in any of the other synthetic tasks.
The reason is that the MILP returns the minimum number of modifications for the $10$ computed plans, invalidating those plans but potentially leaving other valid solutions unaffected.

\begin{figure}[t]
    \centering
    \includegraphics[width=1\linewidth]{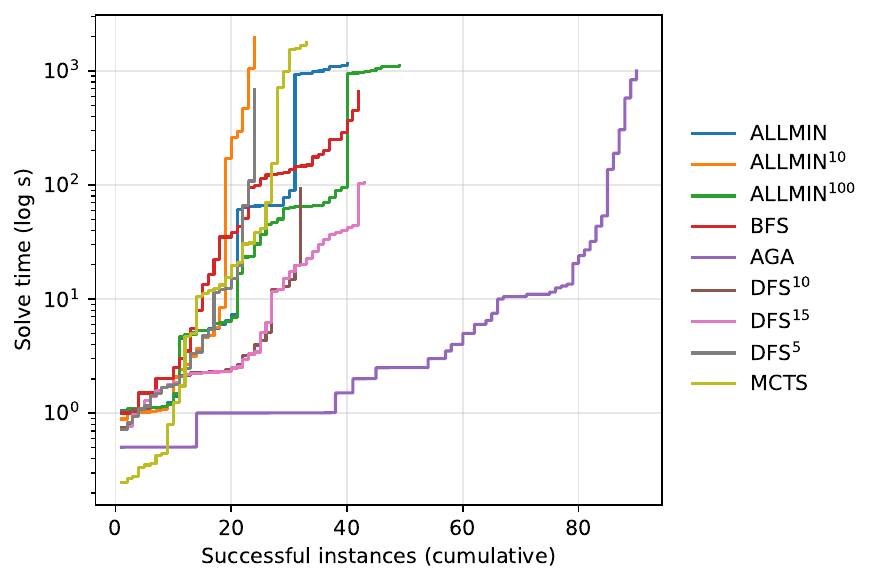}
    \caption{$\compilation$ succeeds on all instances with faster runtime, $\allplansmin^{100}$ and $\dfs^{15}$ follow.}
    \label{fig:success_analysis}
\end{figure}
On the other hand, $\allplansmin$ guarantees to return a (minimal) set of modifications invalidating all plans, but \textsc{symk} often cannot complete this enumeration within the memory and time boundaries.
We analyzed how the execution time is distributed between the two main components of \allplansmin: generating all plans that solve the task using \textsc{symk}, and solving the MILP to determine the necessary modifications.
The execution time of \allplansmin increases exponentially as the size of the planning task and the number of plans that solve it ($|\Pi(\task)|$) grow, rising from a few seconds for smaller tasks to hundreds of seconds for larger ones (runtime details are available in the supplementary material).

\begin{figure*}[t]
    \centering
    \includegraphics[width=1\linewidth]{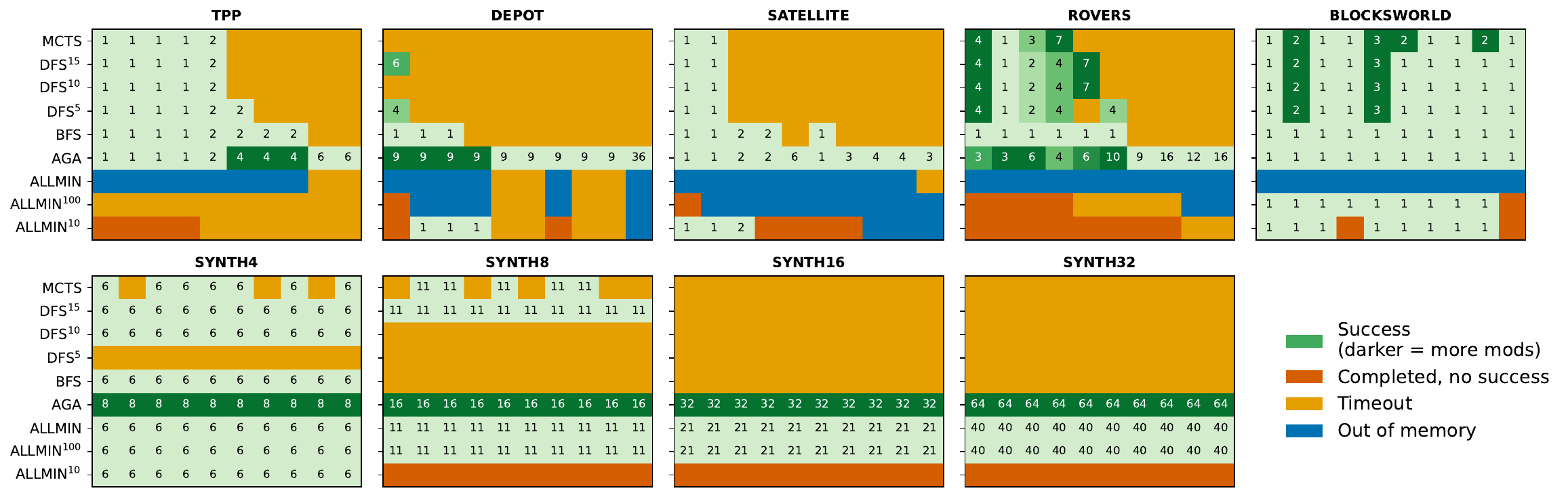}
    \caption{Experimental results on the 90 benchmark problems. Each column denotes one of the 10 problems for each domain.}
    \label{fig:calendario}
\end{figure*}

The planning domains proved to be difficult for \allplansmin and its variants, with $\allplansmin^{10}$ alone scoring some successes in \textsc{satellite} and \textsc{depot} and together $\allplansmin^{100}$ solving most of \textsc{blocksworld}.
For \textsc{depot} and \textsc{tpp} the bottleneck is \textsc{symk}, as \allplansmin reaches the time limit before concluding the collection of the given number of plans. 
On the other hand, in \textsc{satellite}, \textsc{rovers}, \textsc{blocksworld} and the larger \textsc{synth} benchmarks, the time dedicated to solving the MILP dominates the \textsc{symk} runtime. 
This suggests that, for most benchmarks, the planning tasks are easier to solve than the optimization problem, which becomes increasingly complex as the number of actions, fluents, and plans grows, leading to a larger number of variables and constraints.
We observe that computing a subset of all plans often enables $\allplansmin^{10}$ and $\allplansmin^{100}$ to achieve strong empirical performance, although there is no guarantee regarding the success of the returned modifications.
$\allplansmin^{100}$ thus balances the amount of information collected about the plans and the time spent to compute it.

Among the search methods, $\dfs^{15}$ solves the most instances with 43 successful modifications, followed by \bfs at 42. 
While \bfs is optimal, \dfs is faster: \dfs finds a successful set of modifications in 36.4\% less time on average, but with an average of 79.2\% more modifications.
\mcts offers no advantage over \dfs and \bfs in either success rate or number of modifications, showing that plan length and solving time do not represent a reliable signal to guide the exploration of the action modifications.

With respect to the number of modifications, we observe two main groups. 
For the standard planning benchmarks, \bfs is more reliable than \allplansmin at finding an optimal solution, as in many cases a single modification renders the planning task unsolvable ($|M|=1$). 
This occurs because, in many of these tasks, removing a landmark fact or adding preconditions that are never achievable in certain states can make the planning task unsolvable with just one change. 
In contrast, the \textsc{synth} benchmarks are more challenging, as they are designed with several independent paths leading to the goal and fewer bottlenecks.
Starting from \textsc{synth}8, the number of modifications required (11) is sufficiently large for \bfs to run out of time exploring all possible combinations of modifications up to $|M|=11$, while $\dfs^{15}$ and \mcts can still solve many \textsc{synth}8 problems. 
However, \textsc{synth}16 and \textsc{synth}32 require a number of modifications beyond reach for the search methods, while \textsc{symk} can complete the plan enumeration quickly, allowing \allplansmin to compute an optimal set of modifications for all \textsc{synth} problems.
$|M|$ is always lower than the number of plans, indicating that \allplansmin effectively identifies actions shared across multiple plans and modifies them so that several plans become invalid simultaneously.

We also inspected and compared some of the modifications proposed by the algorithms.
We noted some modifications that exploit the possibility of rendering some actions inapplicable by adding as preconditions ground facts that can never be satisfied in the domain.
For example, in \textsc{blocksworld}, adding the fluent $on(d,d)$ to any action precondition precludes the use of the action in any plan. 
This type of modifications can be prevented by restricting the set $\mathcal{M}$ of allowed modifications.
Another example of modification that stretches the domain semantics is the addition of a fluent $f$ as precondition of an action $a$ that has the same $f$ as effect: when $a$ is the only action that can add $f$, the precondition can never be satisfied and $a$ can never be used.
For example, in \textsc{satellite}, the action $a=\mathit{calibrate\_satellite0\_instrument0\_groundstation2}$ 
has as effect $f=\mathit{calibrated(instrument0)}$ and the modification proposed by \mcts is to add $f$ as a precondition of $a$. Note that this is equivalent to removing $f$ from the effects of $a$, which is what \bfs and \dfs return in this specific instance.

\section{Conclusions and Future Work}
In this paper, we introduce planning task shielding: the problem of identifying the plans that lead to a flawed state in the original planning task and then automatically making the task unsolvable by minimally modifying the original set of actions.
We formalize this problem and show that it can be addressed by considering only a subset of possible modifications to the original actions: adding preconditions, removing add effects, and adding delete effects.
We then present three different approaches to solve these tasks, each aiming to minimize the number of modifications to the original actions that render the planning task unsolvable.
Our evaluation demonstrates that our proposed approaches can effectively render planning tasks unsolvable, thereby shielding the system and preventing the existence of plans that reach flawed states.

Our current problem formulation consider granular modifications at the grounded action level.
As next steps, we want to explore the lifted version of the problem, where we reason about modifying action schemes rather than grounded actions. 
We would also like to consider additional objectives, such as minimizing the number of actions to which modifications are applied, or minimizing the number of fluents used in the modifications.

\section*{Disclaimer}
This paper was prepared for informational purposes by the Artificial Intelligence Research group of JPMorgan Chase \& Co. and its affiliates ("JP Morgan'') and is not a product of the Research Department of JP Morgan. JP Morgan makes no representation and warranty whatsoever and disclaims all liability, for the completeness, accuracy or reliability of the information contained herein. This document is not intended as investment research or investment advice, or a recommendation, offer or solicitation for the purchase or sale of any security, financial instrument, financial product or service, or to be used in any way for evaluating the merits of participating in any transaction, and shall not constitute a solicitation under any jurisdiction or to any person, if such solicitation under such jurisdiction or to such person would be unlawful.
© 2026 JPMorgan Chase \& Co. All rights reserved

\bibliography{aaai27}

\clearpage

\twocolumn[
  \centering
  {\Large\bfseries Supplementary Material\par}
  \vspace{2.0em}
]

\section{Plots}

Figure~\ref{fig:allplansmin_time} plots the average runtime breakdown for \allplansmin between the time dedicated to generating the given number of valid plans using \textsc{symk}, and solving the MILP to determine the necessary modifications. Bars with higher transparency denote the runs completed but not successful, i.e. returning a set of modifications $|M|$ such that the modified problem remains solvable. Bars exceeding the timeout are due to \textsc{symk} returning a plan after the given timeout expired. 

\begin{figure*}[h!]
    \centering
    \includegraphics[width=1\linewidth]{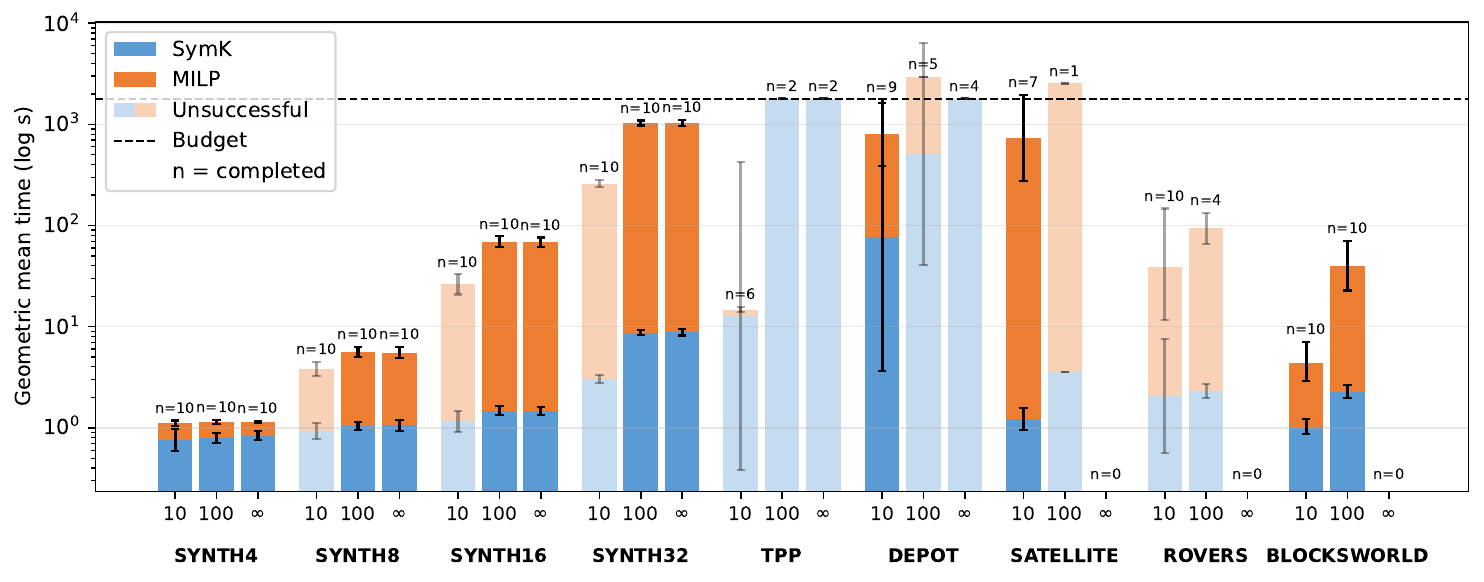}
    \caption{Execution time split into the time to compute the set of plans with \textsc{symk} (blue) and the time to compute the minimum number of modifications with the MILP (orange).}
    \label{fig:allplansmin_time}
\end{figure*}

Figure~\ref{fig:modcounts} shows the distribution of the number of modifications across methods and domains: \compilation on average returns a higher number of modifications across all benchmarks, while \allplansmin and \bfs are optimal.

\begin{figure*}[h!]
    \centering
    \includegraphics[width=1\linewidth]{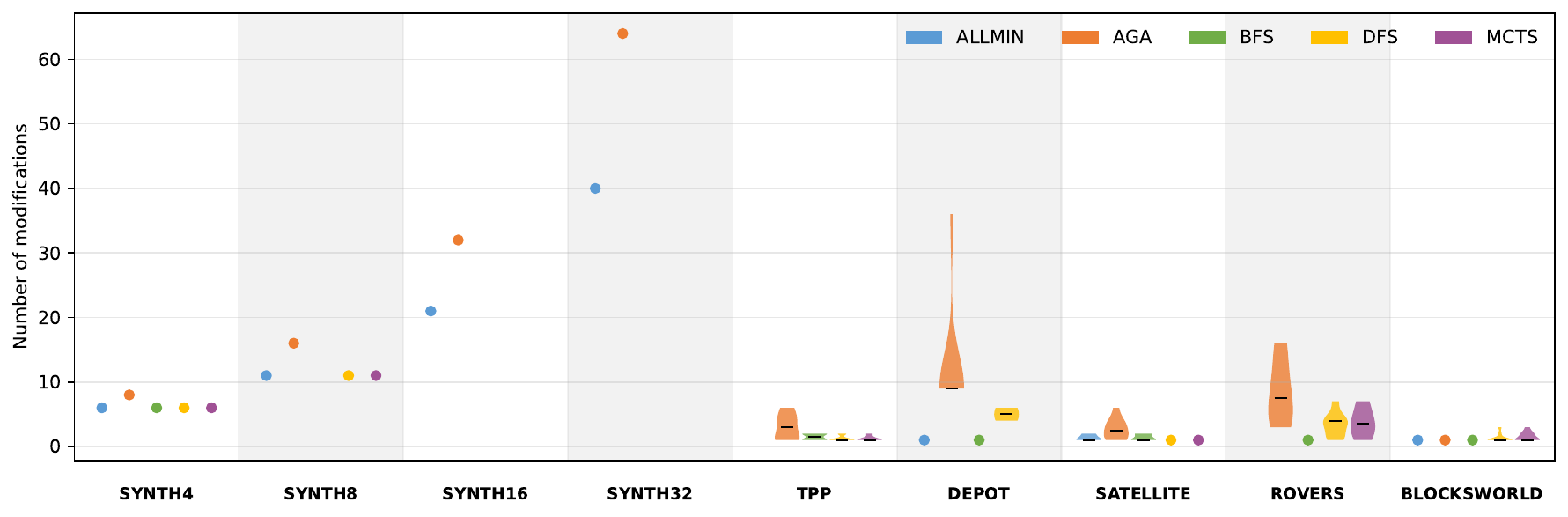}
    \caption{Number of modifications per benchmark and method. While \compilation successfully renders unsolvable all benchmarks, the average number of modifications is higher.}
    \label{fig:modcounts}
\end{figure*}

Figure~\ref{fig:runtime_analysis} shows the average runtimes for each method on the successful problems.

Note that both the mean number of modifications and mean runtime on some domain may not be extrapolated from the same set of problems.

\begin{figure*}[h!]
    \centering
    \includegraphics[width=1\linewidth]{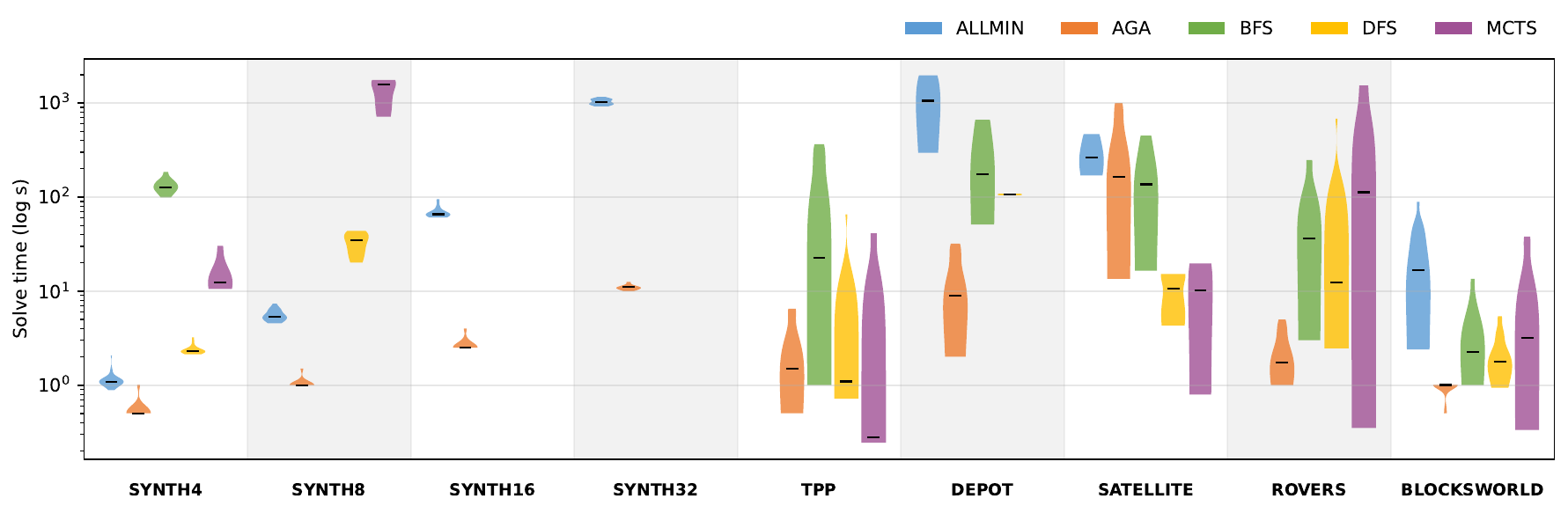}
    \caption{Mean execution time per benchmark and method. \compilation is consistently faster, while \mcts shows the higher variance across the search-based methods.}
    \label{fig:runtime_analysis}
\end{figure*}

\section{Benchmarks statistics}

Table~\ref{tab:results} presents the methods performance on the 9 benchmarked domains. As noted before, the average number of modifications and runtime is comparable across methods only when the set of successful problem is the same. Refer to Figure 2 in the main paper for a problem by problem comparison.

\begin{table*}[h!]
\centering
\resizebox{\textwidth}{!}{%
\begin{tabular}{l|r@{\hspace{3pt}}r@{\hspace{3pt}}rcc|r@{\hspace{3pt}}r@{\hspace{3pt}}rcc|r@{\hspace{3pt}}r@{\hspace{3pt}}rcc}
\toprule
Config & \multicolumn{5}{c}{SYNTH4} & \multicolumn{5}{c}{SYNTH8} & \multicolumn{5}{c}{SYNTH16} \\
 & T & C & S & $|M|$ & Time & T & C & S & $|M|$ & Time & T & C & S & $|M|$ & Time \\
\midrule
{\allplansmin}$^{10}$ & $10$ & $10$ & $\mathbf{10}$ & $\mathbf{6.0 \pm 0.0}$ & $1.1 \pm 0.3$ & $10$ & $10$ & $0$ & - & - & $10$ & $10$ & $0$ & - & - \\
{\allplansmin}$^{100}$ & $10$ & $10$ & $\mathbf{10}$ & $\mathbf{6.0 \pm 0.0}$ & $1.1 \pm 0.1$ & $10$ & $10$ & $\mathbf{10}$ & $\mathbf{11.0 \pm 0.0}$ & $5.6 \pm 0.7$ & $10$ & $10$ & $\mathbf{10}$ & $\mathbf{21.0 \pm 0.0}$ & $69.6 \pm 10.1$ \\
\allplansmin & $10$ & $10$ & $\mathbf{10}$ & $\mathbf{6.0 \pm 0.0}$ & $1.1 \pm 0.1$ & $10$ & $10$ & $\mathbf{10}$ & $\mathbf{11.0 \pm 0.0}$ & $5.6 \pm 0.8$ & $10$ & $10$ & $\mathbf{10}$ & $\mathbf{21.0 \pm 0.0}$ & $68.6 \pm 8.7$ \\
\compilation & $10$ & $10$ & $\mathbf{10}$ & $8.0 \pm 0.0$ & $\mathbf{0.6 \pm 0.2}$ & $10$ & $10$ & $\mathbf{10}$ & $16.0 \pm 0.0$ & $\mathbf{1.1 \pm 0.2}$ & $10$ & $10$ & $\mathbf{10}$ & $32.0 \pm 0.0$ & $\mathbf{2.7 \pm 0.5}$ \\
\bfs & $10$ & $10$ & $\mathbf{10}$ & $\mathbf{6.0 \pm 0.0}$ & $134.2 \pm 23.7$ & $10$ & $0$ & $0$ & - & - & $10$ & $0$ & $0$ & - & - \\
{\dfs}$^{5}$ & $10$ & $0$ & $0$ & - & - & $10$ & $0$ & $0$ & - & - & $10$ & $0$ & $0$ & - & - \\
{\dfs}$^{10}$ & $10$ & $10$ & $\mathbf{10}$ & $\mathbf{6.0 \pm 0.0}$ & $2.4 \pm 0.3$ & $10$ & $0$ & $0$ & - & - & $10$ & $0$ & $0$ & - & - \\
{\dfs}$^{15}$ & $10$ & $10$ & $\mathbf{10}$ & $\mathbf{6.0 \pm 0.0}$ & $2.4 \pm 0.2$ & $10$ & $10$ & $\mathbf{10}$ & $\mathbf{11.0 \pm 0.0}$ & $33.4 \pm 8.3$ & $10$ & $0$ & $0$ & - & - \\
\mcts & $10$ & $7$ & $7$ & $\mathbf{6.0 \pm 0.0}$ & $15.8 \pm 7.3$ & $10$ & $5$ & $5$ & $\mathbf{11.0 \pm 0.0}$ & $1340.4 \pm 461.5$ & $10$ & $0$ & $0$ & - & - \\
\bottomrule
\end{tabular}}
\par\medskip
\resizebox{\textwidth}{!}{%
\begin{tabular}{l|r@{\hspace{3pt}}r@{\hspace{3pt}}rcc|r@{\hspace{3pt}}r@{\hspace{3pt}}rcc|r@{\hspace{3pt}}r@{\hspace{3pt}}rcc}
\toprule
Config & \multicolumn{5}{c}{SYNTH32} & \multicolumn{5}{c}{TPP} & \multicolumn{5}{c}{DEPOT} \\
 & T & C & S & $|M|$ & Time & T & C & S & $|M|$ & Time & T & C & S & $|M|$ & Time \\
\midrule
{\allplansmin}$^{10}$ & $10$ & $10$ & $0$ & - & - & $10$ & $6$ & $0$ & - & - & $10$ & $9$ & $3$ & $\mathbf{1.0 \pm 0.0}$ & $1107.4 \pm 838.2$ \\
{\allplansmin}$^{100}$ & $10$ & $10$ & $\mathbf{10}$ & $\mathbf{40.0 \pm 0.0}$ & $1032.4 \pm 62.2$ & $10$ & $2$ & $0$ & - & - & $10$ & $5$ & $0$ & - & - \\
\allplansmin & $10$ & $10$ & $\mathbf{10}$ & $\mathbf{40.0 \pm 0.0}$ & $1030.8 \pm 81.2$ & $10$ & $2$ & $0$ & - & - & $10$ & $4$ & $0$ & - & - \\
\compilation & $10$ & $10$ & $\mathbf{10}$ & $64.0 \pm 0.0$ & $\mathbf{11.0 \pm 0.7}$ & $10$ & $10$ & $\mathbf{10}$ & $3.0 \pm 2.1$ & $2.4 \pm 2.2$ & $10$ & $10$ & $\mathbf{10}$ & $11.7 \pm 8.5$ & $\mathbf{13.0 \pm 10.8}$ \\
\bfs & $10$ & $0$ & $0$ & - & - & $10$ & $8$ & $8$ & $1.5 \pm 0.5$ & $112.8 \pm 149.4$ & $10$ & $3$ & $3$ & $\mathbf{1.0 \pm 0.0}$ & $296.3 \pm 322.8$ \\
{\dfs}$^{5}$ & $10$ & $0$ & $0$ & - & - & $10$ & $6$ & $6$ & $1.3 \pm 0.5$ & $13.7 \pm 25.9$ & $10$ & $1$ & $1$ & $4.0 \pm 0.0$ & $108.5 \pm 0.0$ \\
{\dfs}$^{10}$ & $10$ & $0$ & $0$ & - & - & $10$ & $5$ & $5$ & $\mathbf{1.2 \pm 0.4}$ & $3.3 \pm 5.0$ & $10$ & $0$ & $0$ & - & - \\
{\dfs}$^{15}$ & $10$ & $0$ & $0$ & - & - & $10$ & $5$ & $5$ & $\mathbf{1.2 \pm 0.4}$ & $4.3 \pm 7.3$ & $10$ & $1$ & $1$ & $6.0 \pm 0.0$ & $104.7 \pm 0.0$ \\
\mcts & $10$ & $0$ & $0$ & - & - & $10$ & $5$ & $5$ & $\mathbf{1.2 \pm 0.4}$ & $8.5 \pm 18.3$ & $10$ & $0$ & $0$ & - & - \\
\bottomrule
\end{tabular}}
\par\medskip
\resizebox{\textwidth}{!}{%
\begin{tabular}{l|r@{\hspace{3pt}}r@{\hspace{3pt}}rcc|r@{\hspace{3pt}}r@{\hspace{3pt}}rcc|r@{\hspace{3pt}}r@{\hspace{3pt}}rcc}
\toprule
Config & \multicolumn{5}{c}{SATELLITE} & \multicolumn{5}{c}{ROVERS} & \multicolumn{5}{c}{BLOCKSWORLD} \\
 & T & C & S & $|M|$ & Time & T & C & S & $|M|$ & Time & T & C & S & $|M|$ & Time \\
\midrule
{\allplansmin}$^{10}$ & $10$ & $7$ & $3$ & $1.3 \pm 0.6$ & $299.1 \pm 152.1$ & $10$ & $8$ & $0$ & - & - & $10$ & $10$ & $8$ & $\mathbf{1.0 \pm 0.0}$ & $4.4 \pm 2.0$ \\
{\allplansmin}$^{100}$ & $10$ & $1$ & $0$ & - & - & $10$ & $4$ & $0$ & - & - & $10$ & $10$ & $9$ & $\mathbf{1.0 \pm 0.0}$ & $40.2 \pm 21.8$ \\
\allplansmin & $10$ & $0$ & $0$ & - & - & $10$ & $0$ & $0$ & - & - & $10$ & $0$ & $0$ & - & - \\
\compilation & $10$ & $10$ & $\mathbf{10}$ & $2.7 \pm 1.6$ & $317.6 \pm 361.2$ & $10$ & $10$ & $\mathbf{10}$ & $8.5 \pm 4.9$ & $\mathbf{2.3 \pm 1.4}$ & $10$ & $10$ & $\mathbf{10}$ & $\mathbf{1.0 \pm 0.0}$ & $\mathbf{1.0 \pm 0.2}$ \\
\bfs & $10$ & $5$ & $5$ & $1.4 \pm 0.5$ & $178.0 \pm 178.7$ & $10$ & $6$ & $6$ & $\mathbf{1.0 \pm 0.0}$ & $73.6 \pm 91.0$ & $10$ & $10$ & $\mathbf{10}$ & $\mathbf{1.0 \pm 0.0}$ & $4.1 \pm 4.0$ \\
{\dfs}$^{5}$ & $10$ & $2$ & $2$ & $\mathbf{1.0 \pm 0.0}$ & $9.9 \pm 7.3$ & $10$ & $5$ & $5$ & $3.0 \pm 1.4$ & $145.8 \pm 300.4$ & $10$ & $10$ & $\mathbf{10}$ & $1.3 \pm 0.7$ & $2.3 \pm 1.4$ \\
{\dfs}$^{10}$ & $10$ & $2$ & $2$ & $\mathbf{1.0 \pm 0.0}$ & $\mathbf{9.7 \pm 7.5}$ & $10$ & $5$ & $5$ & $3.6 \pm 2.3$ & $28.1 \pm 37.0$ & $10$ & $10$ & $\mathbf{10}$ & $1.3 \pm 0.7$ & $2.3 \pm 1.3$ \\
{\dfs}$^{15}$ & $10$ & $2$ & $2$ & $\mathbf{1.0 \pm 0.0}$ & $10.7 \pm 6.4$ & $10$ & $5$ & $5$ & $3.6 \pm 2.3$ & $30.0 \pm 41.8$ & $10$ & $10$ & $\mathbf{10}$ & $1.3 \pm 0.7$ & $2.3 \pm 1.3$ \\
\mcts & $10$ & $2$ & $2$ & $\mathbf{1.0 \pm 0.0}$ & $10.3 \pm 13.4$ & $10$ & $4$ & $4$ & $3.8 \pm 2.5$ & $439.8 \pm 732.1$ & $10$ & $10$ & $\mathbf{10}$ & $1.5 \pm 0.7$ & $9.8 \pm 13.9$ \\
\bottomrule
\end{tabular}}
\caption{Outcome breakdown and modification counts per method and benchmark. T= total problems; C=completed, S=success. T-C=out of memory or timeout, C-S=modifications do not render the task unsolvable.}
\label{tab:results}
\end{table*}

\end{document}